\documentclass{article}

\usepackage{ijcai13}
\usepackage{epsfig,graphicx}
\usepackage{verbatim}
\usepackage{amsfonts}

\long\def\comment#1{}

\newcommand{\D}{\:\: |\:\:}
\newtheorem{definition}{Definition}
\newtheorem{theorem}{Theorem}

\newtheorem{proposition}[theorem]{Proposition}

\newtheorem{examp}{Example}
\newenvironment{proof}{{\bf Proof:}}{$\Box$} 
\newenvironment{example}{\begin{examp}\rm}{\end{examp}}

\begin{document}

\title{Majority Rule for Belief Evolution in Social Networks}

\author{}
\date{}
\maketitle

\begin{abstract}
In this paper, we study how an agent's belief is affected by her
neighbors in a social network. We first introduce a general
framework, where every agent has an initial belief on a statement,
and updates her belief according to her and her neighbors' current
beliefs under some belief evolution functions, which, arguably,
should satisfy some basic properties. Then, we focus on the majority
rule belief evolution function, that is, an agent will (dis)believe
the statement iff more than half of her neighbors (dis)believe it.
We consider some fundamental issues about majority rule belief
evolution, for instance, whether the belief evolution process will
eventually converge. The answer is no in general. However, for
random asynchronous belief evolution, this is indeed the case.
\end{abstract}

\section{Introduction}

How agents change their beliefs is a fundamental problem in
Artificial Intelligence. Traditionally in the area of Knowledge
Representation and Reasoning, this problem is normally formalized in
some logical formalism, e.g., propositional logic, and minimal
change serves as the first principle
\cite{Alchourron-etal85,KatsunoMendelzon89}.

In this paper, we consider this problem from a social aspect, that
is, how an agent's belief is affected by her neighbors in a social
network. For instance, an agent might form a belief that the share
price of IBM will increase tomorrow when she discussed this issue
with her colleagues. However, she might also change this belief,
i.e., to believe that the price will actually decrease, after
convinced by her family members later on.

In fact, this aspect, also known as opinion formation and social
learning, has been widely studied in other relevant fields such as
mathematics and statistics
\cite{DeGroot74,Holley75,Liggett85,Krause00}, economics
\cite{Ellison93,DeMarzo03,sandholm2010,AcemogluOP10,Golub10},
sociology \cite{Friedkin90,HegselmannK02}, biology \cite{Clifford73}
and so on. Recently, it has also attracted some attentions in
theoretical computer science \cite{BindelKO11}. However, as far as
we know, it has been long neglected in the AI community.

Consider a multiagent context, where a social community is formed by
some agents. This is usually modeled by a graph as a social network,
where each node represents an agent and each edge represents a
social tie between two agents. Now consider the agents' beliefs
about a particular statement, e.g., whether the share price of IBM
will increase tomorrow. An agent may form a belief about this
statement based on some observations and evidences, e.g., the
financial situation of IBM recently and the trading volume of the
stock today. However, her belief is also heavily affected by other
agents, in particular, her friends' opinions on this statement.

In this paper, we first introduce a general framework to model this
phenomenon. We consider the agents' beliefs about a particular
statement in a social network. Initially, each agent has a prior
belief, which might be formed based on her own experiences,
observations and evidences. Then, the agents start to communicate
each other synchronously or asynchronously, and update their beliefs
according to their neighbors' current beliefs in the social network
under some belief evolution function. We argue that these belief
evolution function, although can be defined in many different ways,
should satisfy a list of basic properties.

Next, we focus on the majority rule as the belief evolution
function. That is, an agent will (dis)believe the statement iff more
than half of her neighbors (dis)believe it. Majority rule is one of
the most natural and well studied function in related fields such as
voting \cite{Gaertner09}. It does make sense under our context as
well. For instance, in the IBM share price scenario, if a majority
number of traders believe that the stock price will
increase/decrease, this will likely be the case.

We investigate some fundamental properties about this function.
Among them, a key issue is the convergence of the belief evolution
process, i.e., will all evolution sequences be stable eventually. It
can be observed that this is not the case in general. However, we
show that, for random asynchronous belief evolution, this is indeed
the case.

\section{Belief Evolution in Social Networks}

We consider a multiagent context. A society is formed by some agents
that are connected. Formally, a {\em social network} $N$ is a
directed graph $\langle A, T\rangle$, where $A$ is a set of {\em
actors} (also called {\em individuals} or {\em agents}) in a
society, and $T \subseteq A \times A$ is a set of dyadic {\em ties}
(also called {\em relationships} or {\em connections}) among agents.
We assume that $(a,a) \in T$ for all $a \in A$ as an agent must know
herself. We say that $b$ is connected to $a$ if $(b,a) \in T$, and
connected by $a$ if $(a,b)\in T$. For convenience, we simply use $a
\in N$ to denote $a \in A$.

Now we consider the agents' beliefs about a statement $s$, for
instance, whether the share price of IBM will increase tomorrow. The
agents' beliefs might be formed according to many different reasons.
We group them into two categories: evidence based influences and
communication based influences. The former includes some facts such
as the financial situation of IBM in the last three months while the
latter includes some facts such as the opinions of the agents'
friends on this statement.

In this paper, we separate the influences of these two categories
into two steps. First, all the agents form an initial belief on the
statement based on the evidence based influences. Then, the agents
start to communicate each other to update their beliefs. We are
mainly focused on the latter step, called {\em belief evolution}.

Once formed an initial belief, at a certain time point, all the
agents' beliefs on the statement can be viewed as a belief profile.
\begin{definition}[Belief profile]\label{profile}
Let $N=\langle A, T\rangle$ be a network, and $s$ a statement. A
{\em belief profile} $P$ of $N$ on $s$ is a mapping from $A$ to
$\{0,1\}$, i.e., $P: A \mapsto \{0,1\}$.
\end{definition}
Here, $P(a), a \in A$ is the opinion of agent $a$ on the statement
$s$. Particularly, $P(a)=1$ means that the agent $a$ believes $s$,
while $P(a)=0$ means that $a$ disbelieves $s$, i.e., $a$ believes
$\lnot s$.

Given a statement $s$ and a network $N$, we say that a belief
profile $P$ of $N$ on $s$ is a {\em consensus} iff all agents have
the same belief, i.e., either for all $a \in N$, $P(a)=1$ or for all
$a$, $P(a)=0$. We say that two profiles are {\em isomorphic} if
there is a one-to-one correspondence between them. We say that a
profile $P$ is {\em less or equal than} another profile $P'$,
denoted by $P \le P'$, if for all $a \in N$, $P(a) \le P'(a)$. We
use $\overline{P}$ to denote a new profile obtained from $P$ by
flipping over all beliefs, i.e., for all $a \in N$,
$\overline{P}(a)=1-P(a)$.

In this paper, we only consider a single statement $s$. This is
because the influences of other statements on $s$ are categorized as
evidence based, and their influences are taken into account in the
agents' initial beliefs. Hence, we omit $s$ in the belief profile
$P$ if it is clear from the context. Also, we assume that the
network structure is fixed throughout the paper. Hence, we sometimes
omit $N$ in the belief profile $P$ as well.

We can visualize a belief profile $P$ as a labeled graph based on
the network $N$ with a label on each node, either $1$ or $0$,
indicating this agent's opinion on the statement.

\comment{\footnote{The social network is generally a directed graph.
However, for simplicity and clarity, we use the undirected version
in examples.}}

\begin{example} \label{profile-example}
Figure \ref{fig1} depicts four belief profiles of the same network
$N_0$. Here, $P_1$ is a consensus while the rest are not; $P_2$ and
$P_3$ are isomorphic; also $P_2=\overline{P_3}$; $P_4 \le P_1$ but
it is not comparable with $P_2$.
\begin{figure}
\begin{center}
\includegraphics[width=8cm]{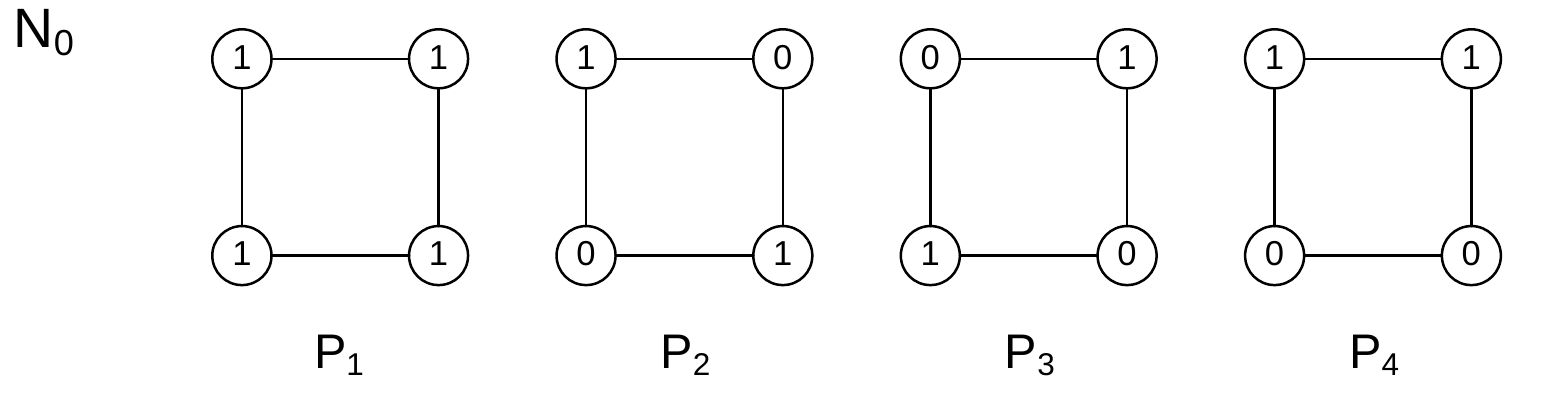}
\caption{Belief profiles of the network $N_0$} \label{fig1}
\end{center}
\end{figure}
\end{example}

The belief profile defines all agents' opinions on the statement at
a certain time point. At this time point, the agents will
communicate with other agents and reconsider their beliefs. The
reconsideration is based on their own strategies.

\begin{definition}[Belief evolution function]\label{evolution-function}
Let $N=\langle A, T\rangle$ be a social network, $a \in A$ an agent
and $s$ a statement. A {\em belief evolution function} $f$ of $a$ on
$s$, denoted by $f_a$, is a mapping $f_a: \mathfrak{P} \mapsto
\{0,1\}$, where $\mathfrak{P}$ is the set of all possible belief
profiles of $N$ on $s$.
\end{definition}
Intuitively, after taking into account the overall belief profile
$P$ on $s$ in the network $N$, the agent $a$ decides to revise her
belief to a new one $f_a(P)$. For convenience, we use $f_N$ to
denote the collection of all belief evolution functions of agents in
$N$ and use $f_N(P)$ to denote the belief profile obtained from $P$
by applying $f_N$ on all agents. In this sense, $f_N$ updates a
belief profile to a new one. Again, $s$ might be omitted in $f_a$
and $f_N$.

The belief evolution function in Definition \ref{evolution-function}
can be defined arbitrarily. For example, a special one is so-called
{\em stubborn}, that is, the agent never changes her belief. On the
contrary, another function is that the agent always flips over her
belief at every step. There are other possibilities, for instance,
an agent will change her belief as soon as there are at least 3 of
her friends taking an opposite opinion.

However, not every belief evolution function is rational, for
instance, the one that always flipping over her beliefs. Of course,
we are only interested in those rational ones. For this purpose, we
propose a list of desirable properties. Let $N=\langle A, T\rangle$
be a network and $a \in A$. We say that a belief evolution function
$f_a$ is
\begin{description}
\item[bounded] if $min\{P(b) \D b \in N \} \le f_a(P) \le max\{P(b) \D b \in N
\}$ for all profiles $P$.
\item[neutral] if $f_a(\overline{P})=1-f_a(P)$ for all profiles $P$.
\item[congruent] if $f_a(P)=f_{a'}(P')$ for any two isomorphic profiles
$P$ and $P'$, where $a$ and $a'$ are corresponded.
\item[local] if for any two profiles $P$ and $P'$ that agree $s$ for
all agents connected by $a$ (including $a$ herself), i.e., for all
$b \in N$ such that $(a,b) \in T$, $P(b)=P'(b)$, we have
$f_a(P)=f_a(P')$.
\item[monotonic] if $P \le P'$ implies that $f_a(P) \le f_a(P')$ for all profiles $P$ and $P'$.
\item[non-slavish] if there does not exist an agent $b$ such that for
all profiles $P$, $f_a(P)=P(b)$.
\end{description}
We say that the collection $f_N$ of all believe evolution functions
is bounded (neutral, congruent, local, monotonic and non-slavish) if
for all $a \in N$, $f_a$ is bounded (neutral, congruent, local,
monotonic and non-slavish).

We argue that a rational belief evolution function should satisfy
these properties. Boundedness means that the agent will not change
her belief if all agents reach a consensus. Neutrality means that
the function will not be affected by how the statement is
represented, for instance, from $s$ to $\lnot s$. Congruency means
that the function will not be affected by how the social network is
represented. Locality means that the agent only has local
information, that is, the agent is not able to get the beliefs of
those agents not known by her. Monotonicity means that if an agent
changes her belief to a new positive (negative resp.) one under a
circumstance, then in another circumstance that is at least as
positive (negative) as this one, she will do the same thing.
Finally, non-slavishness means that the agent cannot be dominated by
a single agent in any circumstance. In particular, non-slavishness
implies non-stubbornness.

Not every belief function satisfies these properties. For instance,
the stubborn function satisfies all but not non-slavishness. The one
that always flips over beliefs is neutral, congruent, local,
non-slavish but does not satisfy the rest. The one that changes the
belief iff 3 of her friends having an opposite opinion is bounded,
congruent, local, neutral, non-slavish (if connected to at least 3
other agents) but not monotonic.

By applying the belief evolution functions, the belief profile of a
network changes from one to another. We call this a {\em belief
evolution step}. First, we consider the case that all agents have to
reconsider their beliefs at every step, called {\em synchronous
belief evolution}. Given an initial belief profile $P^0$ and the
belief evolution function $f_N$ for all agents in the network $N$,
the synchronous belief evolution will perform iteratively.

\begin{definition}[Synchronous belief evolution]
Let $N=\langle A, T\rangle$ be a network, and $f_N$ a collection of
belief evolution functions for all agents in $N$. Let $P^0$ be an
initial belief profile of $N$ on a statement $s$. The {\em
synchronous belief evolution} for $P^0$ under $f_N$ is a sequence
$\{P^0,\dots,P^i,\dots\}$ of belief profiles, where $P^{i+1}$ is
obtained from $P^i$ by applying $f_N$ as follows:
\[
\textrm{for any } a \in N, P^{i+1}(a)=f_a(P^i).
\]
We simply write $P^{i+1}=f_N(P^i)$.
\end{definition}

\comment{ Note that synchronous belief evolution only depends on
$P^0$ and $f_N$.}

Also, we consider asynchronous belief evolution, in which not all
agents are forced to reconsider their beliefs at a certain time
point. The rationale for asynchronous belief evolution is twofold.
First, different agents may communicate with their friends
asynchronously due to, e.g., the frequency of contact and/or the
strength of their friendship. Second, it is possible that an agent
may sometimes stick on her own belief even she knows her friends'
opinions.

Let $B$ be a subset of agents in a network $N$ and $P$ a belief
profile of $N$. The belief profile obtained from $P$ by applying
$f_N$ on agents in $B$, denoted by $f^B_N(P)$, is
\begin{itemize}
\item $f^B_N(P)(a)=f_a(P)$, for $a \in B$.
\item $f^B_N(P)(a)=P(a)$, for $a \not\in B$.
\end{itemize}

\begin{definition}[Asynchronous belief evolution]
Let $N=\langle A, T\rangle$ be a network, and $f_N$ a collection of
belief evolution functions for all agents in $N$. Let $P^0$ be an
initial belief profile of $N$ on a statement $s$, and
$\sigma=B_1,\dots B_n,\dots$ be a sequence of groups of agents. The
{\em asynchronous belief evolution} for $P^0$ under $f_N$ with
respect to $\sigma$ is a sequence $\{P^0,\dots,P^i,\dots\}$ of
belief profiles, where $P^{i+1}=f_N^{B_{i+1}}(P^i)$.
\end{definition}
Here, $B_i$ means those agents who want to reconsider their beliefs
at time point $i$. In asynchronous belief evolution, $P^{i+1}$ is
obtained from $P^i$, $f_N$ together with $B_{i+1}$. Clearly,
synchronous belief evolution can be regarded as a special case of
asynchronous belief evolution by setting the sequence of groups of
agents as $A,\dots A,\dots$.

\comment{Note that both synchronous and asynchronous belief
evolution satisfy the Markov property in the sense that the new
belief profile only depends on the previous one (but not further
historical ones). }

We are mainly interested in the dynamics of the agents' beliefs. In
the framework, the agents will update their beliefs according to
other agents' opinions. A question is, will this evolution process
stop, if yes, at what kind of belief profiles?

\begin{definition}[Equilibrium]\label{equilibrium}
Let $N$ be a network and $f_N$ the belief evolution functions. We
say that a belief profile $P$ is an {\em equilibrium} under $f_N$ if
$P=f_N(P)$.
\end{definition}

\begin{proposition}\label{equilibrium-same}
That $P$ is an equilibrium under $f_N$ iff for any subset $B$ of
agents, $P=f_N^B(P)$.
\end{proposition}
\begin{proof}
That $P$ is an equilibrium under $f_N$ iff for all agents $a$,
$P(a)=f_a(P)$ iff for any subset $B$ of agents, $P=f_N^B(P)$.
\end{proof}

Proposition \ref{equilibrium-same}, although simple, shows that
equilibrium in terms of synchronous belief evolution is the same as
equilibrium in terms of asynchronous belief evolution.

Obviously, if $f_N$ is bounded, then the consensus profile must be
an equilibrium under any evolution function $f_N$. However, if $f_N$
is not bounded, maybe there exists no equilibrium under $f_N$ at
all. A simple counterexample is that an agent always flips over her
belief.

\comment{Also, as we will show later, there are other equilibrium
profiles under a given network $N$ and function $f_N$. This implies
that, in some cases, consensus may never be reached.}

In this paper, our main concern is whether a belief evolution
process will eventually terminate on an equilibrium.

\begin{definition}[Convergence]\label{termination}
Let $N$ be a network, $f_N$ the belief evolution functions, $P^0$ an
initial belief profile of $N$ and $\sigma$ a sequence of groups of
agents. The belief evolution for $P^0$ under $f_N$ with respect to
$\sigma$ {\em converges} if there exists a number $k$ such that
$P^k$ is an equilibrium of $f_N$. In this case, $P^k$ is called the
{\em convergence} of this belief evolution process. For synchronous
evolution, we simply say that the evolution process for $P^0$ under
$f_N$ converges.
\end{definition}

\comment{Strictly speaking, termination is perhaps a more accurate
name than convergence for Definition \ref{termination} as
convergence also includes the cases to reach an equilibrium
infinitely. However, under our setting, if the social network is
finite, then the number of possible belief profiles (together with
potential subsets of agents) is finite. Hence, we mix the usage of
these two terminologies.}

It can be observed that for some $f_N$, the evolution process will
never converge, e.g., the one that an agent always flips over her
beliefs. Of course, this is an extreme case as the belief evolution
function is not rational.

\section{Majority Rule for Belief Evolution}

\comment{ In the previous section, we have introduced a framework
for belief evolution in social networks. This is a general framework
as the belief evolution functions of agents can be defined
arbitrarily (arguably satisfying some desirable properties).
However, it is not possible to perform the actual belief evolution
process if the functions remain undefined.}

This section dedicates to a natural yet representative belief
evolution function of the framework presented in the previous
section, namely the majority rule function, originated from the
majority rule for voting system \cite{Gaertner09}. For majority rule
evolution function, an agent will change her belief iff a majority
number of her friends (including herself) have an opposite opinion.

\begin{definition}[Majority rule]\label{majority}
Let $N=\langle A, T\rangle$ be a social network and $a \in A$ an
agent. The {\em majority rule} belief evolution function $m_a$ is
defined as
\begin{displaymath}
m_a(P)=\left \{ \begin{array}{ll} 1 & \textrm{if }
N^+(a,P)>N^-(a,P) \\
0 & \textrm{if } N^+(a,P)<N^-(a,P) \\
P(a) & \textrm{if } N^+(a,P)=N^-(a,P),
\end{array}\right.
\end{displaymath}
where $P$ is a belief profile of $N$, $N^+(a,P)=|\{b|(a,b) \in T,
P(b)=1\}|$ and $N^-(a,P)=|\{b|(a,b) \in T, P(b)=0\}|$ respectively.
We use $m_N$ to denote the collection of all majority rule functions
of agents in $N$.
\end{definition}
Here, $N^+(a,P)$ ($N^-(a,P)$) is the number of agents connected by
$a$ and (dis)believing in the statement.

\comment{Note that the majority rule belief evolution function is
just one possible model of the framework presented in the previous
section. In fact, it is homogeneous in the sense that every agent
uses the same function, while, in general, belief evolution
functions for different agent can be defined heterogeneously. }

The majority rule function satisfies all desirable properties
mentioned in the previous section.
\begin{theorem}
The majority rule belief evolution function is bounded, neutral,
congruent, local, monotonic, and is non-slavish if every agent is
connected to at least two other agents.
\end{theorem}
\begin{proof}
For space reasons, we only show that the majority rule function is
monotonic here. Suppose that $P$ and $P'$ are two profiles such that
$P\le P'$. Then, for agent $a$ and every agent $b$ (including $a$)
connected by $a$, i.e., $(a,b) \in T$, we have $P(b) \le P'(b)$.
There are three cases:
\begin{itemize}
\item $m_a(P)=0$. In this case, $m_a(P) \le m_a(P')$.
\item $m_a(P)=1$ and $N^-(a,P) = N^+(a,P)$. In this case, $P(a)=1$ and $N^-(a,P') \le N^-(a,P) = N^+(a,P) \le N^+(a,P')$. Therefore, $P'(a)=1$
and $N^-(a,P') \le N^+(a,P')$. Hence, $m_a(P')=1$ so that $m_a(P)
\le m_a(P')$.
\item $m_a(P)=1$ and $N^-(a,P) < N^+(a,P)$. In this case, $N^-(a,P') \le N^-(a,P) < N^+(a,P) \le
N^+(a,P')$. Hence, $m_a(P')=1$ so that $m_a(P) \le m_a(P')$.
\end{itemize}
\end{proof}

We apply the majority rule function for belief evolution. We are
mainly interested in some fundamental issues related to equilibrium,
convergence and consensus. For instance, given a network, what are
the equilibria under the majority rule function? Does every
evolution sequence converge for any initial belief profile? Can the
consensus be reached eventually?

\comment{Some questions can be answered by simple observations. For
instance, the consensus belief profiles, in which all agents believe
or disbelieve the statement, are equilibria under $m_N$ for any
given network $N$ because majority rule function is bounded. Hence,
there always exists an equilibrium for any social network. But the
answers for other questions are not so obvious, for instance, the
convergence problem. We analyze the majority rule belief evolution
with those questions.}

First of all, let us consider some examples for synchronous belief
evolution.

\begin{example}\label{example-first}
Figure \ref{fig2} depicts the synchronous belief evolution processes
under the majority rule function for two different instances, where
the agents and their initial beliefs are the same, but $N_2$ has an
extra edge than $N_1$. Both evolution processes converge after
several steps.
\begin{figure}
\begin{center}
\includegraphics[width=8cm]{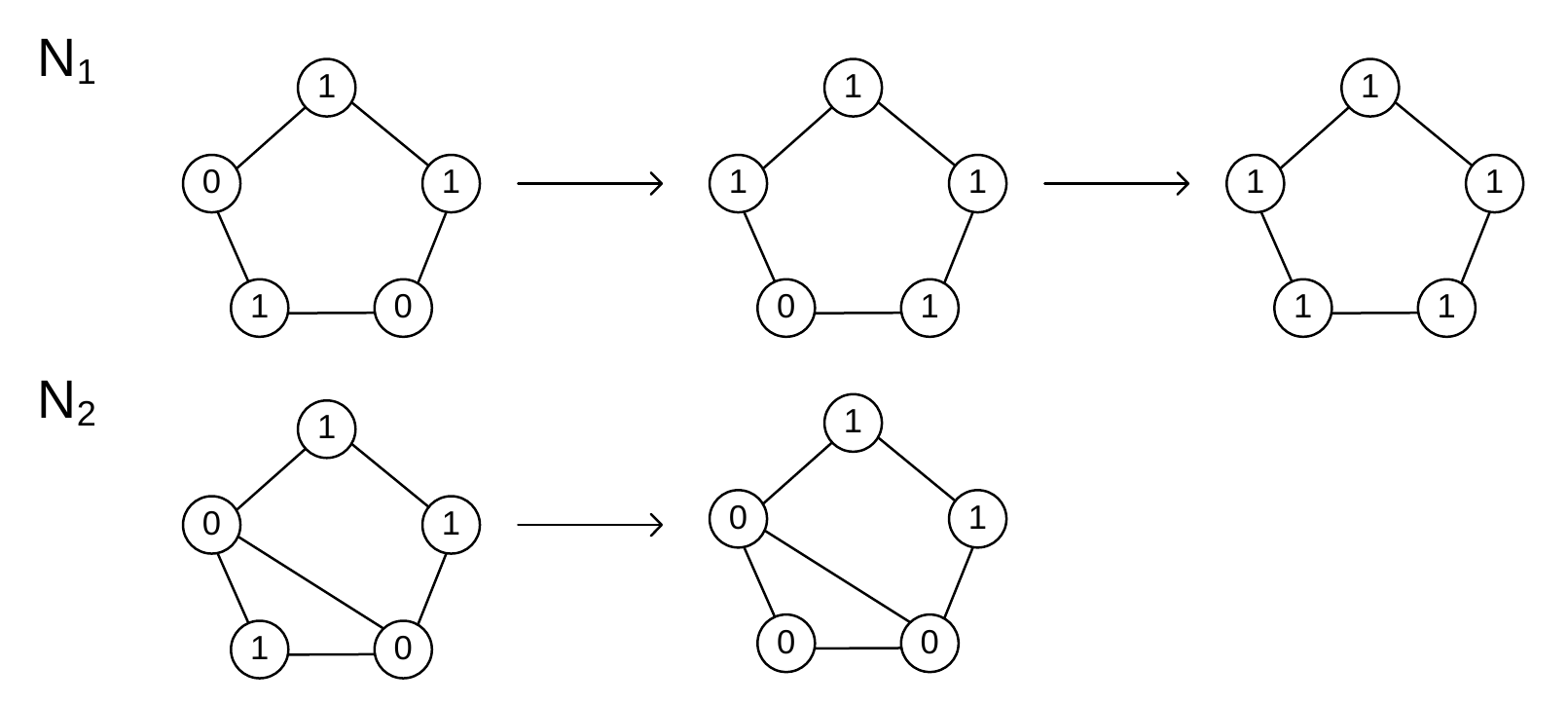}
\caption{An example of synchronous belief evolution} \label{fig2}
\end{center}
\end{figure}
\end{example}

Example \ref{example-first} illustrates that the network structure
is critical for belief evolution. Although $N_2$ has only one more
edge than $N_1$, the resulting equilibria are quite different. While
the convergence of evolution in $N_1$ is a consensus, the one in
$N_2$ is not. In fact, the former is a positive consensus, but the
majority opinion of the latter is, on the contrary, negative.
Observed from the evolution process in $N_2$, the majority opinion
of the convergence may not be the same as that of the initial
profile.

Also, Example \ref{example-first} shows that there might exist
non-consensus equilibria for some social networks. Actually, the
profile $P_4$ in Example \ref{profile-example} is another one. This
shows that, in some situations, consensus may never be reached even
if the belief evolution process converges.

In Example \ref{example-first}, both evolution processes terminate
within several steps. In fact, it can be shown that, the synchronous
belief evolution for any initial belief profile of $N_1$ and $N_2$
must converge. Unfortunately, this is not the case in general.

\begin{example}\label{example-non-converge}
Figure \ref{fig3} depicts the synchronous belief evolution process
for $P_2$ in Example \ref{profile-example} under $m_{N_0}$, which
falls into a loop $P_2 \rightarrow P_3 \rightarrow P_2 \dots
\rightarrow P_3 \rightarrow P_2 \dots$.
\begin{figure}
\begin{center}
\includegraphics[width=8cm]{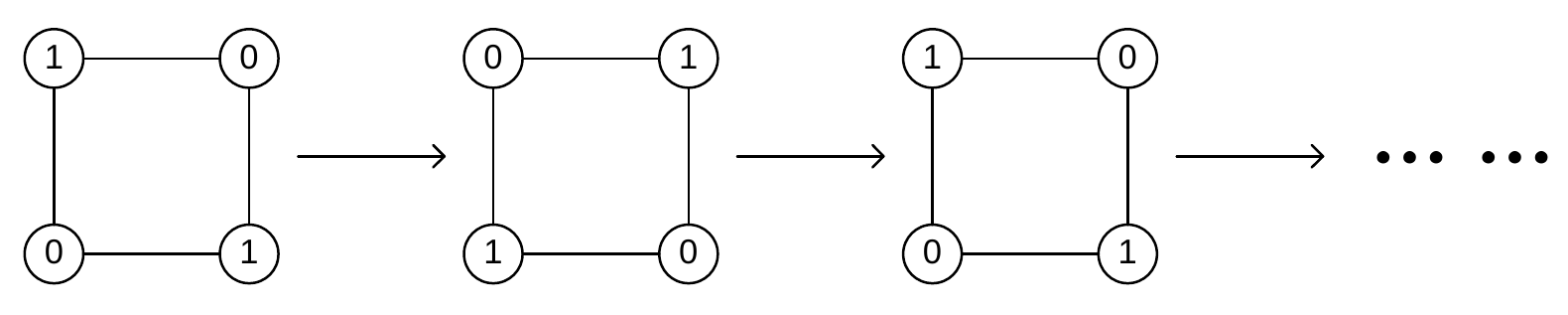}
\caption{Synchronous belief evolution for $P_2$} \label{fig3}
\end{center}
\end{figure}
\end{example}

Observed from Example \ref{example-non-converge}, in some
situations, synchronous belief evolution under the majority rule
function may never converge. If the network is finite, this must
fall into a loop because synchronous belief evolution is
deterministic. In these cases, there is no final result of belief
evolution.

Now we consider asynchronous belief evolution. Of course, not all
asynchronous evolution sequences converge since synchronous belief
evolution is a special case. Instead, we are concerned with whether
a particular sequence of groups of agents will lead to a
convergence. Again, consider the belief evolution for $P_2$ in $N_0$
under the asynchronous context.

\begin{example}\label{example-asynchronous}
Figure \ref{fig4} depicts three different asynchronous evolution
sequences for $P_2$ under $m_{N_0}$. The agents who evolved their
beliefs at the previous round are shadowed. In this example, all
evolution processes converge.
\begin{figure}
\begin{center}
\includegraphics[width=7cm]{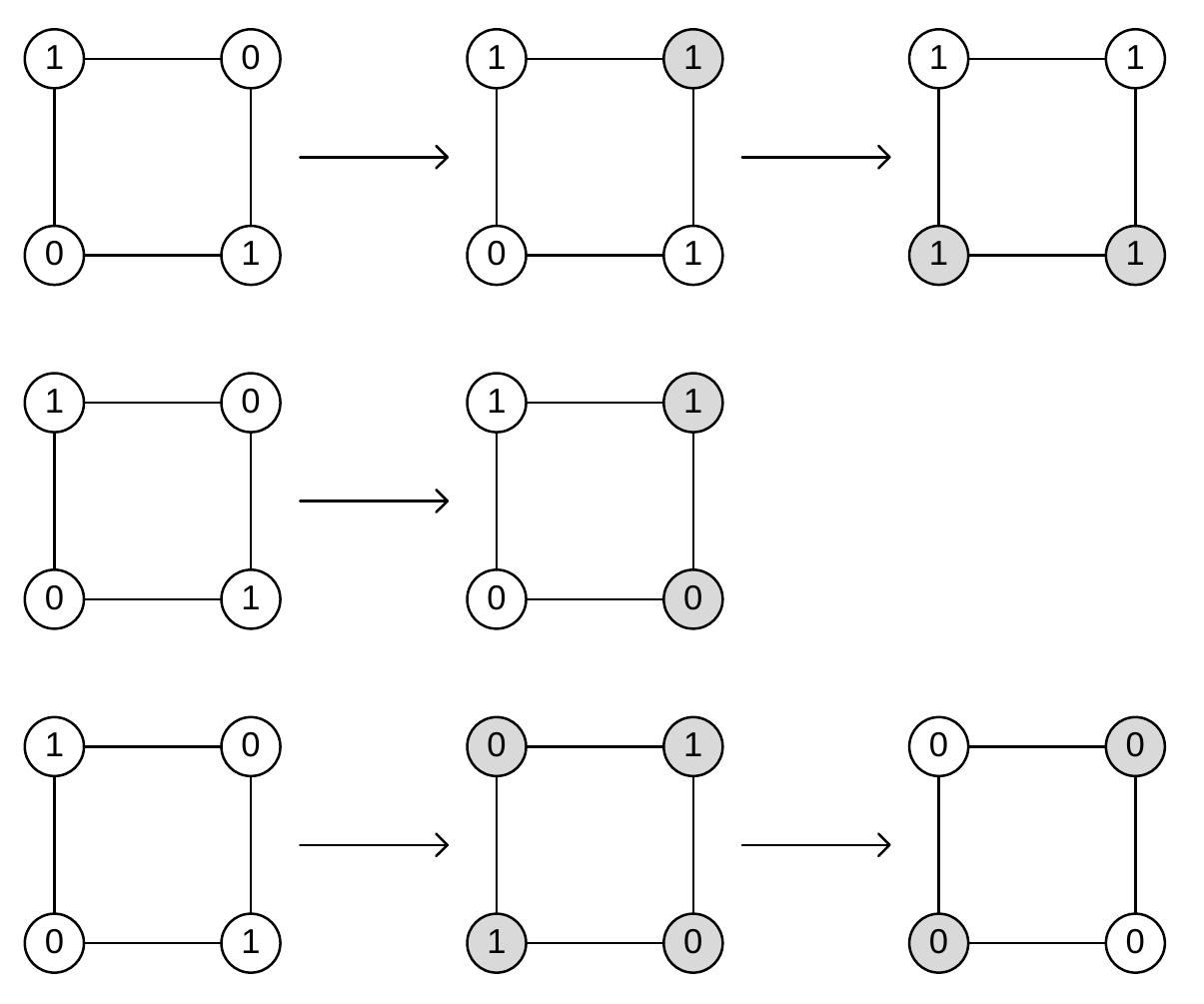}
\caption{Asynchronous belief evolution for $P_2$} \label{fig4}
\end{center}
\end{figure}
\end{example}

Compared to the synchronous belief evolution for $P_2$ in Example
\ref{example-non-converge}, there exists an (actually many)
asynchronous evolution sequence that converges. However, different
asynchronous evolution sequences may lead to exactly opposite
results, for instance, the first and the last sequences in Example
\ref{example-asynchronous}.

The following theorem shows that there always exists a converging
sequence for any initial belief profile under the majority rule
function.

\begin{theorem}\label{convergence-exist}
Let $N=\langle A,T \rangle$ be a finite social network. For any
belief profile $P$ of $N$, there exists a sequence $\sigma$ of
groups of agents such that the asynchronous belief evolution for $P$
under $m_N$ w.r.t. $\sigma$ converges.
\end{theorem}
\begin{proof}
We prove a stronger result that this theorem holds under any
monotonic belief evolution function $f_N$. We directly construct
such a sequence, which is divided into two phrases as follows.
\begin{description}
\item[Increasing phrase] At each round in this phrase, we flip over those negative
beliefs to positive ones (i.e. from $0$ to $1$) whenever possible.
Multiple rounds might be needed. Let $P'$ be the final belief
profile obtained in this phrase. Clearly, $P \le P'$.
\item[Decreasing phrase] On the contrary, at each round in this phrase,
we flip over those positive beliefs to negative ones (i.e. from $1$
to $0$) whenever possible. Similarly, multiple rounds might be
needed. Let $P''$ be the final belief profile obtained in this
phrase. Clearly, $P'' \le P'$.
\end{description}
We prove that $P''$ must be an equilibrium by contradiction. Assume
that there exists $a \in N$ such that $f_a(P'') \ne P''(a)$. Then,
$P''(a)= 0$ and $f_a(P'')=1$ according to the construction of the
decreasing phrase. Otherwise, $a$ will be further selected in the
decreasing phrase so that $P''$ is not the final profile obtained, a
contradiction. There are two cases:
\begin{itemize}
\item $P'(a)=1$. In this case, $a$ must be selected at some round in the
decreasing phrase to flip over from $1$ to $0$. Therefore, there
exists $P^*$ in the decreasing phrase such that $P^*(a)=1$ but
$f_a(P^*)=0$. Clearly, $P'' \le P^* \le P'$. Therefore, $1=f_a(P'')
\le f_a(P^*) =0$, a contradiction.
\item $P'(a)=0$. In this case, $f_a(P')=1$ since $f_a(P'')=1$ and
$P'' \le P'$, which contradicts to our construction that $P'$ is the
final profile obtained in the increasing phrase.
\end{itemize}
This completes our proof.
\end{proof}

However, the existence of converging asynchronous evolution does not
mean this convergence will eventually be reached. Firstly, the
agents themselves do not know which sequence will lead to a
convergence because they do not have global information. Secondly,
even if they know a converging sequence, perhaps they are not
cooperative enough to follow it.

Let us go back to the two major reasons for considering asynchronous
belief evolution. First, different agents may have different
frequency of communication with their friends. Second, agents might
be sometimes over-confident even if they know their friends'
objections. We can use a random variable to simulate both cases.

This motivates us to consider random asynchronous belief evolution,
in which the agents randomly evolve their beliefs at each round.
More precisely, each agent is associated with a random Boolean
variable to decide whether or not she will evolve her belief at a
certain time point. The value of the random variable will be
determined according to a probability distribution. For instance, at
the current stage, agent $a$ might have $0.8$ chance of evolving her
belief while agent $b$ might only have $0.4$. Hence, at a certain
time point, a subset of agents will be generated according to the
random variables, which is the set of agents who will evolve their
beliefs at the current stage. In this sense, random asynchronous
belief evolution can be regarded as a statistical process of
asynchronous evolution. We argue that random asynchronous belief
evolution is more realistic than synchronous belief evolution and
asynchronous belief evolution based on intentionally selected
sequences.

The following theorem shows that random asynchronous belief
evolution will eventually converge.

\begin{theorem}\label{random-evolution}
For any finite social network and initial belief profile, random
asynchronous belief evolution under the majority rule function
always converges.
\end{theorem}
\begin{proof}
To prove this, we need to introduce a notion called belief profile
transition graph. Let $N$ be a social network and $f_N$ the belief
evolution functions for agents in $N$. The {\em belief profile
transition graph} for $f_N$ is an edge labeled graph $\langle
\mathcal{P}, \mathcal{T} \rangle$, where $\mathcal{P}$ is the set of
all belief profiles of $N$, and $\mathcal{T}$ is the set of
transitions among profiles. Each edge is labeled with a subset of
agents in $N$. For two profiles $P$, $P'$ and a subset $B$ of
agents, $(P,P') \in T$ labeled by $B$ iff $P'=f_N^B(P)$, i.e., $P'$
is the belief profile obtained from $P$ by applying $f_N$ on agents
in $B$.

Now we consider the strongly connected components of the belief
profile transition graph for the majority rule function. First of
all, each equilibrium itself forms a strongly connected component.
Consider the condensation of the transition graph, which is a
directed acyclic graph. On the one hand, by Theorem
\ref{convergence-exist}, a node in the condensation is a leaf (i.e.
no outgoing edge) iff it is an equilibrium under the majority rule
function. On the other hand, if the network is finite, then for any
initial belief profile, random asynchronous belief evolution will
eventually lead to a leaf in the condensation because any subset of
agents might be generated by randomness for any profile. This shows
that random asynchronous belief evolution eventually converges.
\end{proof}

To end up this section, we summarize the main observations and
results for belief evolution under the majority rule function.
\begin{itemize}
\item The network structure is critical for belief evolution.
\item For some networks, there exists non-consensus
equilibrium. As a consequence, consensus may never be reached in
belief evolution.
\item For synchronous belief evolution, the evolution process may
never converge for some initial belief profiles.
\item The majority opinion of the convergence might not be
the same as the one of the initial profile.
\item For asynchronous belief evolution, there always exists an
evolution sequence leading to a convergence for any initial belief
profile. In some cases, there might exist many, and the resulting
final equilibria could be quite different.
\item Random asynchronous belief evolution processes always
converge.
\end{itemize}

\section{Related Work}

Belief and opinion dynamics in social network is widely studied in
related fields. Several rule of thumbs models are proposed in the
literature
\cite{DeGroot74,Friedkin90,Krause00,HegselmannK02,DeMarzo03,Ellison93,Golub10,AcemogluOP10}.
In DeGroot's \shortcite{DeGroot74} seminal work, each agent has an
initial opinion (a continuous value), and iteratively evolves her
belief by taking a weighted sum of her neighbors' opinions. Our
framework shares some similar ideas but differs from it on the
following aspects. First, we intend to propose a framework rather
than a model, in which the belief evolution function can be defined
arbitrarily under some restrictions. Yet we also consider a
representative model of this framework. Second, in our framework,
the value is discrete rather than continuous because we are
concerned with beliefs rather than opinions. Finally, in DeGroot's
model, the agents evolve their beliefs simultaneously, while we
consider both synchronous and asynchronous belief evolutions.

Another highly related work is from the area of statistical
mechanics, particulary interacting particle systems
\cite{Clifford73,Holley75}. In the voter model, each agent has an
initial belief (either true or false) and randomly evolves her
belief according to their neighbors' beliefs under a transition
function, which should satisfy some properties as well. However, the
transition function in the voter model calculates a probability
rate, based on which the agent will flip over her beliefs. In
contrast, the belief evolution function in our framework directly
calculates the value of the belief.

The majority rule evolution function is named from the same well
known approach in voting system \cite{Gaertner09}. Generally
speaking, belief evolution can be considered as voting in a social
network for two opposite candidates. However, in belief evolution,
the majority rule voting is performed locally, individually,
distributedly and iteratively, while in voting system, it is
performed globally, wholly, contralizedly and only once.

There are other related works, actually from several different
disciplines. For instance, an alternative model of opinion formation
is to take just a single (but not all) friend's opinion into account
based on their contact frequency \cite{AcemogluOP10}. Another
interesting approach, called replicator dynamics
\cite{sandholm2010}, takes the historical performance of beliefs
into account, and the belief that performs better in the past will
be more likely replicated. However, for space reasons, we are not
able to discuss all of them in details.

\section{Conclusion}

In this paper, we considered the problem of belief evolution in
social networks. To sum up, the main contributions are as follows.
\begin{itemize}
\item We introduced a general framework for belief evolution in
social networks. In this framework, the agents form an initial
belief profile on a statement based on evidences and observations,
and then start to communicate each other in the social network to
update their beliefs according to their belief evolution functions.
The belief evolution process is performed iteratively, either
synchronously or asynchronously.
\item We argued that a rational belief evolution function should
satisfy some desirable properties such as boundedness and
monotonicity.
\item We focused on the majority rule belief evolution function, which satisfies all properties mentioned
above. The main discoveries for majority rule belief evolution are
summarized in the end of Section 3.
\item In particular, we focused on the convergence problem. For synchronous belief evolution,
the process may never terminate. For asynchronous belief evolution,
we show that there always exists a converging evolution sequence for
any initial profile. More interestingly, random asynchronous belief
evolution, arguably corresponding to belief evolution in the
reality, always converges.
\end{itemize}

\bibliography{ref}
\bibliographystyle{named}

\end{document}